\documentclass[pmlr]{jmlr}

\usepackage{longtable}

\usepackage{booktabs}
\usepackage[load-configurations=version-1]{siunitx} 

\theorembodyfont{\upshape}
\theoremheaderfont{\scshape}
\theorempostheader{:}
\theoremsep{\newline}

\title[Hypotheses testing on infinite random graphs]{Hypotheses testing on infinite random graphs}

 \author{\Name {Daniil Ryabko} \Email {daniil@ryabko.net} \\
        \addr {INRIA, 40 av. de Halley, 59650 Villeneuve d'Ascq, France}}

 \jmlryear{2017} 
 \jmlrworkshop{Algorithmic Learning Theory 2017} 
 \editors{Steve Hanneke and Lev Reyzin}

\begin{document}

\maketitle

%

%
 \let\emptyset\varnothing
 \let\epsilon\varepsilon
 \let\phi\varphi
 
 \def\ER{Erd\"os-R\'enyi }
 
 \def\dloc{d_{\operatorname{loc}}}
 
 \def\I{\mathbb I}
 \def\N{\mathbb N}
 \def\R{\mathbb R}
 \def\Q{\mathbb Q}
 \def\cG{\mathcal G}
 \def\cP{\mathcal P}
 \def\cS{\mathcal S}
 \def\cE{\mathcal E}
 \def\cX{\mathcal X}
 \def\cY{\mathcal Y}
 \def\cF{\mathcal F}
 \def\M{\mathcal M}
 \def\T{\mathcal T}
 \def\bX{\mathbf X}
 \def\bY{\mathbf Y}
 \def\S{\mathcal S}
 \def\F{\mathcal F}
 \def\E{\mathbb E}
 \def\H{\mathcal H}
 \def\bH{\mathbf H}
 \def\argmin{\operatorname{argmin}}
 \def\argmax{\operatorname{argmax}}
 \def\supp{\operatorname{supp}}
 \def\cl{\operatorname{cl}}
 \def\as{\text{a.s.}}
 
%
%
%
%
%

%

 \begin{abstract}
  Drawing on some recent results that provide the  formalism necessary to definite stationarity for infinite random graphs, this paper  initiates the study of statistical and learning questions pertaining to these objects. Specifically, a criterion for the existence of a  consistent test for complex hypotheses is presented, generalizing the corresponding results on time series. As an application, it is shown how one can test that a tree has the Markov property, or, more generally, to estimate its memory. 
 \end{abstract}
\section{Introduction}
Huge, world-wide graphs permeate our lives. These graphs carry enormous amount of information which is vital for many applications to analyse and exploit.
Some of the questions one often seeks to answer have the form:  is my model for the graph correct?  others,  does this graph have this or that property? Following the approach of classical statistics, this kind of questions can be formulated as {\em testing a complex hypothesis}: asking whether the distribution that generates the data (here, the given graph) belongs to a set of distributions $H_0$ (the null hypothesis) versus this distribution belongs to a set $H_1$ (the alternative). 

This formulation presents an important caveat, which is obvious yet often overlooked: for the approach to work, the distribution generating the data has to belong either to $H_0$ or to $H_1$.  If not, the answer of the test will be arbitrary and completely useless. In other words, if the test gives the answer $H_\delta$, where $\delta$ is 0 or 1, this answer should be interpreted as ``$H_\delta$ is true, {\em or} a low-probability event happened, {\em or} neither $H_0$ nor $H_1$ is true''.  The mathematical results concern the low-probability event; the third possibility is left to be excluded at a different level, that is to say, informally, from application specifics alone.
 What it means in practice is that {\em the alternative $H_1$ has to be sufficiently general as to assure that it captures the distribution of the data  in case the null is wrong}. 
In other words, it should be a general, qualitative hypothesis, whose validity is self-evident on the application level. 

  It appears that the existing approaches fail to provide this kind of generality. In fact, $H_1$ is often a rather specific model (a set of distributions), such as \ER or, to take a  more general example, stochastic block model; see, for example, 
 \citep{arias2014community,bubeck2016testing} and references. $H_0$ is then a hypothesis about the parameters of the model. Such methods may be applied to graphs  for which $H_1$ is obviously wrong; it is clear, for example, that no social network is an \ER graph. 

It is instructive to take a look at classical statistics where the hypotheses testing formulation originates, and to appreciate the fact that the situation is radically different in this field. Thus, classical alternative hypotheses, such as that the distribution that generates the sample at hand is Gaussian, or, more generally, that the samples are independent and identically distributed, already serve a great variety of applications. Indeed, many distributions appearing in nature can reasonably supposed to be Gaussian, and sampling procedures do often give independent and identically distributed samples~--- at least to the extent to which one can formally argue about natural phenomena at all. Yet in other applications, these assumptions break, and more general alternatives are in order. Thus, stock market data, human-generated or biological texts, as well as many other data sources produce samples that are not independent at all. A more general alternative (which is, of course, not universal either) is that the distribution of the sample is stationary. This is perhaps the most general assumption used in statistics; its meaning can roughly be described as that it does not matter whether we have started to record observations at time 0 or at some other time index.  

The goal of this paper is to bring this kind of generality to hypotheses testing on random graphs, and to initiate the study of the questions that emerge. This is possible thanks to the work \citep{Benjamini:12}, which, based on the foundations laid out in \citep{Aldous:07,lyons2016probability}, defines stationarity of infinite random graphs  and opens the way to use basic facts from ergodic theory on these objects.  However, the authors stop short of considering any problems of estimation, learning or statistical analysis. Here we take the first steps towards filling this gap. After laying down some further definitions, we transfer some fundamental results from hypothesis testing on stationary time series to the new formalisms. The main result is a criterion for the existence of consistent tests for hypotheses concerning infinite random graphs; the criterion is applied to show that some relatively simple hypotheses can or cannot be tested. As one of the applications, it is shown that it is possible to test that an infinite random tree is Markov or has memory of a given order. 

\noindent{\bf Related work: property testing.} A rather different approach to testing hypotheses about graphs is known as {\em property testing}. This approach, initiated in  \citep{goldreich1998property} (with some ideas going back to \citep{blum1993self,rubinfeld1996robust}), considers finite, fixed, deterministic, combinatorial objects, such as graphs, but randomized procedures. The question is whether a graph has a certain property.  A testing procedure samples only a small amount of the graph data, and is supposed to return the correct answer with high probability. While the graph has a finite size ($n$ vertices), the test is supposed to sample a portion which is independent of $n$, which means that the graph can be arbitrarily large (one can say, practically infinite). This is a very ambitious goal. 
The weak point, however, is again the kind of alternative hypothesis considered. The null $H_0$ is of course just the set of all graphs that have the property of interest. The alternative is the set of all graphs that are at least $\epsilon$ far from all the graphs in $H_0$. The underlying distance is that of modifying up to $\epsilon$ elements of the graph. Two models are typically considered: dense graphs, where the number of edges is $\Theta(n^2)$, and bounded-degree graphs. The distance changes accordingly: removing $\epsilon n^2$ edges in the former case and $\epsilon n$ in the latter. In either case, there is a potentially vast and unruly buffer zone between $H_0$ and $H_1$, in which the answer of the test is essentially arbitrary.   The main results concern characterizing those properties for which efficient tests exist. In particular, \citep{alon2008characterization} shows that all so-called hereditary properties (those that are  closed under removal of vertices) are testable. See  \citep{ron2010algorithmic,goldreich1997property} for an overview of the results in this area.



%

\section{Setup: Infinite random graphs}\label{s:su}
The formal setup in this section is mainly after \citep{Benjamini:12}, see also \citep{Benjamini:11,Aldous:07}, with some differences that will be pointed out. 
A graph $G=(V_G,E_G)$ is a pair of a set of vertices $V_G$ and a set of edges $E_G$. 
The vertex set $V_G$ is finite or countably infinite. A graph can be either directed or undirected, so the set $E_G$ is either that of ordered or unordered
pairs of vertices (this distinction is irrelevant for the formalisms used here); we do not consider multi-edges or self-loops.
In this work we assume all graphs to have degree bounded by a constant $M\in\N$.
 A {\em network} is a graph 
$G$ together with a map $m$ from $V_G$ to a finite set  $\cX$  
of {\em data points} (also called marks \citep{Aldous:07}), so that 
each vertex $a\in V_G$ is associated with a data point $X_a\in\cX$. 

We are thinking of the data points as of, say, the content of web pages or the (complete) data associated with 
a user of a social network. 

A rooted graph (network) is a pair $(G,o)$ of a graph (network) and a vertex $o\in V_G$. Two graphs
 $(G,o)$ $(G',o')$ are isomorphic  if there is a pair of bijections $\phi_V:V_G\to V_{G'}$ and $\phi_E:E_G\to E_{G'}$  
such that for any edge $e=(v_1,v_2)\in E_G$ we have  $\phi_E(e)=(\phi_V(v_1),\phi_V(v_2))$, and, additionally, the mapping
preserves the roots:  $\phi_V(o)=o'$. For networks we also require that the data points 
associated with $v$ and $\phi_V(v)$ are the same for all $v\in G_V$. In other words, isomorphisms are just
relabellings of a graph. The isomorphism relation (for graphs as well as for networks) is denoted $\simeq$. We generally do not distinguish between a graph and its isomorphism class.

The structure of a probability space on networks can be introduced as follows.
For a rooted graph $(G,o)$ denote $B_{G}(o,r)$  a ball in $G$ of radius $r$ around $o$, that is, the induced subgraph of $G$ consisting of all vertices 
of graph distance at most $r$ from $o$. 
Define the radius $r(G)$ of $(G,o')$ as the maximal graph distance between 
any vertex in $G$ and $o'$. 
For a finite rooted network $(G,o')$ define the set $F_G$ of rooted networks as 
$F_G:=\{ (G,o): \exists r\in\N\ B_{G}(o,r)\simeq (G,o')\}$. 
Let also $r(F_G):=r(G)$.
Note that we consider only  graphs $G$ whose degree is bounded by $M$, so that there are, for each $r\in\N$, only finitely 
many different sets (isomorphism classes) $B_{G}(o,r)$. Since the set $\{F_G: G\text{ is a finite rooted network }\}$ is countable,
we can use the notation $(F_i)_{i\in\N}$. It is easy to see that $(F_i)_{i\in\N}$ is a standard basis of a probability space.
Denote this probability space $(\cG_*,\cF)$ where $\cF$ is the sigma algebra generated by $(F_i)_{i\in\N}$.
Similarly, one can define a standard probability space on the set of networks with a distinguished walk $o_1,o_2,\dots$  (rather than just one root $o$);
denote this probability space $(\cG_\to,\cF_\to)$. Here $\cF_\to$ is the sigma algebra generated by the sets $\F^i_\to, i\in\N$  of 
all finite networks with a distinguished  walk of length $i$.

 A difference with the setup in \citep{Benjamini:12} is that  we use an explicit standard basis, rather than just saying that the probability 
space is Polish. To do so, we require a known upper bound $M$ on the degree of the possible networks we consider, rather than saying that they are all of bounded degree.
This will ensure the compactness with respect to the distributional distance introduced shortly. See also \citep{Benjamini:11} for a similar metric on the set of networks  resulting in a compact metric space.
The space of all bounded-degree random networks (without a fixed bound $M$) is Polish but is not compact.

Denote $\cP(\cG_*)$ and $\cP(\cG_\to)$  the set of all probability distributions on $(\cG_*,\cF)$ and $(\cG_\to,\cF)$ respectively.

For distributions on a standard probability space the distributional distance \citep{Gray:88} is defined as follows. 
For $\mu_1,\mu_2\in \cP(\cG_*)$  (or $\mu_1,\mu_2\in\cP(\cG_\to)$) let 
\begin{equation}\label{eq:dd}
d(\mu_1,\mu_2):=\sum_{k\in\N} w_k|\mu_1(F_k)-\mu_2(F_k)| 
\end{equation}
(in the case of $\cG_\to$ the sum is over ${F^k}_\to$),
where $w_k$ is a summable sequence of positive real weights which is assumed fixed throughout the paper, e.g.\ $w_k:=1/k(k+1)$.
 
The facts about distributional distance relayed here apply to any standard probability space, and are not specific to $\cG_*$ and $\cG_\to$.
A detailed exposition can be found in \citep{Gray:88}.
Specifically, we will use the following.

\begin{proposition}
The spaces $\cG_*$  and $\cG_\to$ are compact with respect to the topology of the distributional distance.
\end{proposition}

One can define the {\em shift operator} on the set of networks with distinguished walk:
 $\theta: (G,(o_n)_{n\ge1})\to(G,(o_{n+1})_{n\ge1})$.
Call a measure $\mu_\to$  stationary if it is invariant under $\theta$: for any $A\in\cF_\to$ we have $\mu(A)=\mu(\theta^{-1}(A))$. 
A stationary distribution $\mu$ is {\em ergodic} if $\mu_\to$ is ergodic for $\theta$, that is, any shift-invariant set has probability 0 or 1.
Denote $\cS'$ ($\cE'$) the set of all stationary (ergodic) measures in $\cP(\cG_\to)$.

Again, the facts about stationary and ergodic distributions that we use apply to any dynamical system defined over a standard probability space \citep{Gray:88}.
Specifically, we will use the {\em ergodic theorem} and {\em ergodic decomposition} introduced below.

Given a network $(G,(O_n)_{n\ge1})$ and $k\in\N$ define the frequency  of occurrence of $F_k$ in $O_1,\dots,O_N$ as
$\nu_N(F_k)={1\over N}\sum_{i=1}^{N}{\mathbb I}_{(G,O_i)\in F_k}$. 

\begin{proposition}[Ergodic theorem]\label{th:erg}
 Let $\mu$ be a stationary ergodic distribution on $(\cG_\to,\cF_\to)$. For any $k\in\N$ we have 
$$
 \nu_N(F_k)\to\mu((G,o)\in\F_k)\as
$$
\end{proposition}

The space $\cP(\cG_\to)$ can be endowed with a structure of probability space induced by the distributional distance.
Denote $\cP(\cP(\cG_\to))$ the  resulting space of probability measures.
This allows us to formulate the following statement known as {\em ergodic decomposition}. Informally, it means
that any stationary measure can be represented as first selecting, according to some probability distribution, a
stationary ergodic measure, and then using that distribution to generate the data (here, a network and a random walk on that network).
\begin{proposition}[Ergodic decomposition]\label{th:dec}
For any $\mu\in\cS'$ there exists a measure $W_\mu\in \cP(\cP(\cG_\to))$  such that $W_\mu(\cE')=1$
and $\mu(B)=\int d W_\mu(\mu)\mu(B)$ for any $B\in\cF_\to$.
\end{proposition}
\begin{proposition}
 The set $\cS'$ is a closed subset of $\cP(\cG_\to)$, and therefore compact.
\end{proposition}

So far, stationarity and ergodicity have been defined for measures on $\cG_\to$, that is, for measures generating
an infinite graph with a distinguished walk. What we want, however, is stationarity and ergodicity for measures
generated infinite (rooted) graphs. This is done using simple random walks as follows \citep{Benjamini:12}.

Given a graph $(G,o)$, one can consider a simple random walk $(O_n)_{n\in\N}$ on $G$ starting at $o$ and proceeding, on each time step, to the neighbour of the current vertex chosen uniformly at random.
 To any probability measure $\mu$ on $\cG_*$ there corresponds a probability measure
$\mu_\to$ on $\cG_\to$, defined by taking a simple random walk on $G$ starting at $o$, where $(G,o)$ is generated by $\mu$.
A distribution $\mu$ on $\cG_*$ is called {\em stationary (ergodic)} if $\mu_\to$ is stationary (ergodic).
Another (equivalent, see \citep{Benjamini:12}) definition of stationarity is to to say that $(G,o)=(G,O_n)$ in distribution 
for all $n\in\N$ (or, equivalently, for $n=1$). 
That is, the distribution is invariant under
 re-rooting along a simple random walk.

Denote $\cS$ ($\cE$) the set of all stationary (ergodic) measures $\mu\in\cP(\cG_*)$.

 \cite{Benjamini:12} define stationarity and ergodicity on $\cG_*$ directly (via random walks as above), rather than passing through
notions on $\cG_\to$ first (the space $\cG_\to$ is still needed, but not $\cS'$ and $\cE'$). This makes the shift operator $\theta$ stochastic. Here we have chosen  deterministic $\theta$ in order to be able to use  the theory about dynamical systems, specifically, the ergodic decomposition. 

Now the ergodic theorem above applies to measures on $\cG_*$ as well. For ergodic decomposition, we need additionally the following simple observation.
\begin{proposition}
 The set $\cS$ is a closed subset of $\cS'$. For any $\mu\in\cS$ the measure $W_\mu$ whose existence is asserted in Proposition~\ref{th:dec} (the ergodic decomposition), is concentrated on $\cE$, that is, $W_\mu(\cE)=1$.
\end{proposition}
\begin{proof}
 For the first statement, take any convergent sequence $\mu_n\in\cS$, $n\in\N$. It converges to  a measure $\mu\in\cS'$, that is, to a 
stationary measure on $\cS_\to$. We need to show that $\mu\in\cS$, that is, the distribution of the distinguished walk $O_1,O_2,\dots$ is that 
of a random walk. For that it is enough to show that $O_2$ is selected uniformly at random from the neighbours of $O_1=o$. 
This follows from the fact that the distribution of $\mu_i(B_G(o,1))$ converges to $\mu(B_G(o,1))$, since for each $\mu_i$ the distribution of $O_2$
is as desired. The second statement is by construction: by definition, the distribution $\mu$ is obtained by taking a random walk 
starting at the root of a (random) rooted graph $(G,o)$. Thus the probability of the event $\{(G,o):\text{ given the distribution of $(O_n)_{n\in\N}$ conditionally on $(G,o)$ is that of a simple random walk}\}$ is 1,
so its  probability is 1 with respect to $W_\mu$-almost every measure. 
\end{proof}

\section{Sampling, estimation}
Next we need to define some way to sample an infinite (random) network. 
There are many ways of doing this; for example, having a large connected subnetwork $G_n$ that grows with time $n$, or
sampling nodes according to some a priori distribution. 
For now we define a sampling scheme of the former type that is most easy to analyse given the definitions above. This definition is  based on a simple random walk.

To be more precise, let $(G,o)$ be an infinite random network, generated by a distribution $\mu$,
and let $O:=(O_n)_{n\in\N}$ be a trajectory of a simple random walk over $G$ with $O_1=o$.
For each $k\in\N$ define $s_k(O)$ as the set of all $k$-neighbours of all nodes $(O_n)_{n\in\N}$:
$$s_k(O):=\{v\in G:\exists i\in\N, v\text{ is at graph distance at most $k$ from $O_i$}\}.$$
Similarly define $s_k(O_{1..n})$. 

The sampling scheme is defined as follows. Let $k(n):\N\to\N$ be some  non-decreasing function that goes to infinity with $n$.
Given an infinite random network $(G,o)$ and $n\in\N$, the sampling is based on the set $s_{k(N)}(O_{1..N})$, where 
$(O_n)_{n\in\N}$ is a simple random walk over $G$ with $O_1=o$. That is, we take a random walk of length $n$ starting at $o$, 
and then query all $k(n)$-neighbours of all the vertices visited.  Define also the random variable $S_N(G,o):=s_{k(N)}(O_{1..N})$.

For any $k\in\N$ define  $\hat\mu(F_k)$ as ${1\over {N}}\nu_n(O_{1..{N}},F_k)$ if $r(F_k)\le k(n)$ and $\hat\mu(F_k)=0$ otherwise.
With this definition, only the set  $S_N(G,o)$ is used for constructing the estimates $\hat \mu$.

The empirical estimates of the distributional distance are defined by replacing  (say) $\mu_1(F^i)$ in~\eqref{eq:dd}  with $\hat\mu_G(F_i)$:
$$
d(S_n(G,o),\mu'):=\sum_{k\in\N} w_k|\hat\mu(F_k)-\mu'(F_k)| 
$$

From the ergodic theorem (Proposition~\ref{th:erg}) we can derive the following.
\begin{proposition}
 For any  stationary ergodic distribution $\mu\in\cP(\cG_*)$ generating  a  random network $(G,o)$, we have $$\lim_{n\to\infty}d(S_n(G,o),\mu)=0\ \mu_\to-\as$$
\end{proposition}
Note that the ``almost sure'' statement is with respect to the distribution generating $(G,o)$ and the random walk $(O_N)_{N\in\N}$ that is used for sampling.
\begin{proof}
 Take an $\epsilon>0$ and find a $K\in\N$ such that $\sum_{k>K}w_k<\epsilon$. 
Then for each $k=1..K$ we have $|\hat\mu(F_k)-\mu(F_k)|\le\epsilon$ from some $N_k$ on, as follows from the ergodic theorem. 
Let $N:=\max_{k=1..K}N_k$ and increase $N$ if necessary to have $k(N)>r(F_i)$ for all $i=1..K$. Now using the the definition~\eqref{eq:dd}  from all $n>N$  we have $d(S_n(G,o),\mu)\le 2\epsilon$.
\end{proof}

\section{Testing: consistency and a criterion}
Given a pair of sets $H_i\subset\cE$, $i\in\{0,1\}$ and an infinite rooted 
graph $(G,o)$ generated by a distribution $\mu\in H_0\cup H_1$, we want to test
whether $\mu\in H_0$ versus $\mu\in H_1$ based on the sampling procedure described in the preceding section.

A {\em test} is a family of functions $\psi^\alpha$ indexed by $\alpha\in(0,1)$ that take 
as input a finite network $g$ with a distinguished walk $O_1,\dots,O_N$ and outputs a binary answer,
where the answer $i$ is interpreted as ``the graph was generated by a distribution that belongs to $H_i$.''
We will assume that $(g, O_1,\dots,O_n)$ has the form $S_n(G)$ for some infinite network  $G$, that is, it can be obtained 
via a  sampling procedure described in the previous section. What this means is that for each $i$ the vertex $O_i$ has 
all of its $k(n)$ neighbours in $G$. 

A test $\phi$ makes the {\em Type I} error if it says $1$ while $H_0$ is true,
and it makes {\em Type II} error if it	 says $0$ while $H_0$ is false.

\begin{definition}[consistency]
 Call a  test $\psi^\alpha, \alpha\in(0,1)$    consistent as a test of $H_0$ against $H_1$ if:
\begin{itemize}
\item[(i)] The probability of Type I error is always bounded by $\alpha$: for every $\mu\in H_0$,
every $n\in\N$ and every $\alpha\in(0,1)$ 
$$\mu_\to(\psi^\alpha(S_n(G,o))=1)\le\alpha,$$  and
\item[(ii)]  Type II error is made not more than a finite number of times with probability 1: 
$$\mu_\to(\lim_{n\rightarrow\infty} \psi^\alpha(S_{n}(G,o))=1)=1$$ 
for every  
$\mu\in H_1$ and every $\alpha\in(0,1)$. 
\end{itemize}
\end{definition}

The following theorem is a generalization of the result of  \cite{Ryabko:121c}. With the set-up above, the proof carries over directly from \citep{Ryabko:121c}.
\begin{theorem}[Criterion for the existence of consistent tests]\label{th:asym}
 Let $H_0\subset\mathcal E$. The following  statements are equivalent:
\begin{itemize}
 \item[(i)] There exists a consistent test for $H_0$ against $\mathcal E\backslash H_0$.
 \item[(ii)] The set $H_0$ has probability 1 with respect to ergodic decomposition of every $\mu$ in the closure of $H_0$: 
              $W_\mu(H_0)=1$ for each $\mu\in\cl H_0$.
\end{itemize}
\end{theorem}

\section{Some examples of testable and non-testable properties}
Theorem~\ref{th:asym} allows us to obtain a number of results on the existence of tests for various properties pertaining to both the graph structure and the data at vertices. The criterion it provides turns out to be rather easy to verify.

\begin{proposition}\label{th:zero}
 Let $H_0:=\{P\in\cE : P(F_i)=0,\forall i\in I\}$ where $I\subset\N$  is a finite or countable index set. That is, $H_0$ consists of  consist of all stationary ergodic distributions such that certain finite subgraphs have probability 0.   Then there exists a consistent test for $H_0$ against its complement to $\cE$.
\end{proposition}
\begin{proof}
 The statement follows directly from Theorem~\ref{th:asym}. 
Indeed,  it is enough to check it for  $I$ a singleton, since this property is preserved under taking countable intersections. Moreover, clearly 
the property  $P_i(F_i=0)$ is preserved when taking the limit of distributions $P_i$, and if $P(F_i)=0$ for a stationary distribution then the same must hold 
for all its ergodic components. 
\end{proof}
Note that one cannot  replace $P(F_i)=0$ in the formulation with $P(F_i)\le \alpha$, where  $\alpha\in(0,1)$ is fixed, even though the set $H_0$ is closed.
 The reason is that this set  is not closed under taking ergodic decompositions.

While relatively simple, Proposition~\ref{th:zero} allows us to establish the existence of consistent tests for a great many graph properties, such as  {\em cycle-} or {\em clique-}freeness, and so on. Moreover, in this case it is easy to  construct an actual test whose existence Proposition~\ref{th:zero} establishes: it is enough to reject $H_0$ if any of the elements $F_i$ occurs along the random walk and accept it otherwise.

It is worth mentioning some negative results that carry over from hypothesis testing on stationary time series. Since time series are a special case of infinite random graphs, the negative results apply directly, without any extra proof. 
In particular, the hypotheses of {\em homogeneity}  and {\em independence} concerning a pair of infinite random graphs state that the distributions of the graphs are the same (homogeneity), or are independent (independence).  Thus, from the corresponding results on time series \citep{Ryabko:10discr,Ryabko:17clin} we establish the following.

\begin{proposition}\label{th:nohomo}
 There is no consistent test for the hypotheses of homogeneity or independence.
\end{proposition}

Another set of examples of testable properties is provided in the next section.

\section{Markov trees}
A (rooted) tree is a graph without cycles. Children of a vertex are its neighbours that are farther away from the root. 
 Here we consider only leafless trees, that is, there are no vertices without children; in addition, the data at vertices (marks) is assumed to take values in  a finite alphabet. For a vertex $v$ of a tree denote $c(v)$ its number of children,  $d(v)$ its distance to the root, and $v_0(v),\dots,v_{d-1}(v)$ the path from the root to $v$, where $v_0(v)$ is always the root.   
Let us call stationary distributions on infinite random graphs that produce trees w.p.~1 {\em stationary trees}. 
\begin{definition}[simple trees, Markov trees]
Call a stationary tree {\em simple} if the distributions of children of a vertex $v$ are conditionally independent given its path to the root $v_0(v),\dots,v_{d-1}(v)$.
 Call a simple stationary  tree   {\em $k$-order Markov} if, given the path to the root,  the distribution of the number of children, as well as the distribution of the data $m(v)$ at the vertex, depends only on the last  $k$ nodes in the path: $P(c(v),m(v)| v_0(v),\dots,v_{d-1}(v))=P(c(v),m(v)|v_{d-k}(v)..v_{d-1}(v))$.
\end{definition}

 Denote $M_k$, $k\ge0$ ($\M_k$) the set of all $k$-order stationary (ergodic) Markov trees, $\M_*:=\cup_{k\in\infty}\M_k$ and $ST$ ($ET$) the set of all stationary (ergodic) simple trees.

The case $k=0$, that is, the set of memoryless trees, is the well-studied class known as Galton-Watson trees, e.g.,  \citep{lyons1995ergodic}. These represent  the graph version of  i.i.d.\ time series. A constructive definition of such trees is as follows. Starting with the root, each vertex $v$ has $c(v)$ children, where the random variables $c(v)$ are independent and identically distributed, apart from $c(o)$ specified below, with a certain distribution $p:=(p_1..p_M)$.  Here we also consider marks at vertices, which in this case are i.i.d.\  The marks at vertices are also i.i.d.\ with a fixed distribution. With the definition, the tree is not stationary because the root has one less neighbour than the rest of the vertices. 
As is shown in \citep{lyons1995ergodic}, this is fixed simply by letting $P(c(o)=k+1)=p_k$, that is, shifting by 1 the number of children of the root. Clearly, this construction can be generalized to Markov trees defined above.

A {\em walk down} a stationary tree is a simple random walk from the  root that never goes up. In other words, on each next step the walk proceeds to the vertex selected uniformly at random among all of the children of the current vertex.
Define the {\em walk-down $k$-order entropy} of a stationary tree $T$  as $h_k(T):=\E h(c(v),m(v)|v_{d-k}(v)..v_{d-1}(v))$, where $v$ is $k$th vertex on a random walk  from the root down without backtracking. Clearly, $h_k(T)\ge h_{k+1}(T)$. Define the entropy rate $h_\infty(T):=\lim_{k\to\infty} h_k(T)$.  These definition are a simple generalization of the corresponding notions in information theory, e.g., \citep{Cover:06}.

Observe that  if a tree is stationary,  so is the walk down. Moreover, the process 
\begin{equation}\label{eq:mt}
  \big(c(v_1),m(v_1)\big),\big(c(v_2),m(v_2)\big),\dots
\end{equation}
 is just a stationary time series, 
which is $k$ order Markov if and only if the tree is $k$ order Markov. We can, furthermore, generalize the results on hypothesis testing from time series to trees, using the criterion from Theorem~\ref{th:asym} and the corresponding results for time series.

The reason to consider the walk down rather than the simple random walk of  of Section~\ref{s:su} (that may go up) is that the latter breaks the Markov property of the tree~--- precisely because it may go up.

\begin{corollary}
 For every $k\ge0$ there exists a consistent test for the set $\M_k$ of $k$ order Markov trees against its complement to the set of all stationary ergodic trees $ET\backslash \mathcal M_k$. There is no consistent test for the set of all finite-memory trees $\M_*$ against its complement $ET\backslash \M_*$.
\end{corollary}

A test whose existence is claimed in the corollary  can be obtained simply by applying a corresponding test for time series to the series~\eqref{eq:mt};
such tests can be found in \citep{BRyabko:06a,BRyabko:06b}.

\section{Discussion}
In this work, a new framework for testing hypotheses about random graphs has been proposed. The main feature of the framework is its generality~--- most importantly, the generality of the alternative hypotheses allowed, which can be the complement of the null hypothesis to the set of all stationary infinite random graphs.  Given the setup and definitions, the results are relatively simple: the proofs carry over from the corresponding results on time series. In this sense, these results are low-hanging fruit, which is none the less interesting because of it, and which is of course only available since the framework is new.   The field for future work that opens is large and is 
potentially rich with many results (which may or may not be as simple to obtain). The main directions one can foresee are as follows.  

 The first and the most obvious one is finding out, for various properties of graphs of interest, whether consistent tests exist. We do not list the relevant properties here, since they  abound in the graph literature. When a consistent test does exist, the next question is constructing it. For the hypotheses considered in this paper this turns out to be simple, but this is not at all the case in general, even for time series; the general construction of the test in \citep{Ryabko:121c}, unfortunately,  is of little help for finding practical algorithms. 

Looking further afield, stationary random graphs appear to offer the possibility of studying a variety of learning problems. Perhaps one of the most intriguing is {\bf graph compression}.  Compression is, in  a  certain precise sense, equivalent to learning, and compressing an object to its theoretical limit can be considered learning everything there is to learn in it. For stationary time series, the entropy rate is this theoretical limit to compression, and it is achievable asymptotically (without knowing anything about the distribution of the sequence), see \citep{BRyabko:88,Algoet:92}. For stationary graphs, we now have a notion of entropy and it remains to be seen whether similar results can be obtained. Previous theoretical results on graph compression are limited to much smaller classes of distributions, such as \ER graphs \citep{choi2012compression}, and binary trees \citep{zhang2014universal}.  Other relevant statistical problems studied on stationary time series   that can now be studied for stationary graphs include {\bf clustering}  \citep{Ryabko:10clust,Khaleghi:15clust,Ryabko:17clin} and {\bf prediction}. 

One important difference with respect to time series is the {\bf notion of sampling}. In time series, the sample is simply an initial segment of the sequence. All the results then are with respect to just one dimension of the sample: the length of the sequence (the dimensionality of the space or the size of the alphabet are of course important, but have nothing to do with the sampling). In particular, the consistency is typically asymptotic with  the sample size growing to infinity.  In graphs, a sample available to the statistician can be an essentially arbitrary part of the graph. It can ``grow to infinity'' in a variety of ways, which are clearly not equivalent from the learning point of view. Note that the empirical distribution over a sample  does not necessarily converge to the graph distribution \citep{Aldous:07}. Here we have opted for the notion of sampling that is the most simple one given the setup: sampling along a random walk (or a walk down, in the case of trees).  This, however, may not be the most practical method. It remains to be seen what are the necessary conditions on a sample under which consistency can be obtained. This question relates directly to the problem of prediction: again, in time series, one is trying to predict the next symbol (or several), while for graphs there is no ``next'' vertex; there are many possibilities of what to try to predict, which should lead to different notions of consistency.


\begin{thebibliography}{26}
\providecommand{\natexlab}[1]{#1}
\providecommand{\url}[1]{\texttt{#1}}
\expandafter\ifx\csname urlstyle\endcsname\relax
  \providecommand{\doi}[1]{doi: #1}\else
  \providecommand{\doi}{doi: \begingroup \urlstyle{rm}\Url}\fi

\bibitem[Aldous and Lyons(2007)]{Aldous:07}
David Aldous and Russell Lyons.
\newblock Processes on unimodular random networks.
\newblock \emph{Electron. J. Probab.}, 12:\penalty0 no. 54, 1454--1508, 2007.

\bibitem[Algoet(1992)]{Algoet:92}
P.H. Algoet.
\newblock Universal schemes for prediction, gambling and portfolio selection.
\newblock \emph{The Annals of Probability}, 20\penalty0 (2):\penalty0 901--941,
  1992.

\bibitem[Alon and Shapira(2008)]{alon2008characterization}
Noga Alon and Asaf Shapira.
\newblock A characterization of the (natural) graph properties testable with
  one-sided error.
\newblock \emph{SIAM Journal on Computing}, 37\penalty0 (6):\penalty0
  1703--1727, 2008.

\bibitem[Arias-Castro et~al.(2014)Arias-Castro, Verzelen,
  et~al.]{arias2014community}
Ery Arias-Castro, Nicolas Verzelen, et~al.
\newblock Community detection in dense random networks.
\newblock \emph{The Annals of Statistics}, 42\penalty0 (3):\penalty0 940--969,
  2014.

\bibitem[Benjamini and Curien(2012)]{Benjamini:12}
Itai Benjamini and Nicolas Curien.
\newblock Ergodic theory on stationary random graphs.
\newblock \emph{Electron. J. Probab.}, 17:\penalty0 no. 93, 1--20, 2012.

\bibitem[Benjamini and Schramm(2011)]{Benjamini:11}
Itai Benjamini and Oded Schramm.
\newblock Recurrence of distributional limits of finite planar graphs.
\newblock In \emph{Selected Works of Oded Schramm}, pages 533--545. Springer,
  2011.

\bibitem[Blum et~al.(1993)Blum, Luby, and Rubinfeld]{blum1993self}
Manuel Blum, Michael Luby, and Ronitt Rubinfeld.
\newblock Self-testing/correcting with applications to numerical problems.
\newblock \emph{Journal of computer and system sciences}, 47\penalty0
  (3):\penalty0 549--595, 1993.

\bibitem[Bubeck et~al.(2016)Bubeck, Ding, Eldan, and
  R{\'a}cz]{bubeck2016testing}
S{\'e}bastien Bubeck, Jian Ding, Ronen Eldan, and Mikl{\'o}s~Z R{\'a}cz.
\newblock Testing for high-dimensional geometry in random graphs.
\newblock \emph{Random Structures \& Algorithms}, 2016.

\bibitem[Choi and Szpankowski(2012)]{choi2012compression}
Yongwook Choi and Wojciech Szpankowski.
\newblock Compression of graphical structures: Fundamental limits, algorithms,
  and experiments.
\newblock \emph{IEEE Transactions on Information Theory}, 58\penalty0
  (2):\penalty0 620--638, 2012.

\bibitem[Cover and Thomas(2006)]{Cover:06}
Thomas~M. Cover and Joy~A. Thomas.
\newblock \emph{Elements of information theory}.
\newblock Wiley-Interscience, New York, NY, USA, 2006.
\newblock ISBN 0-471-06259-6.

\bibitem[Goldreich and Ron(1997)]{goldreich1997property}
Oded Goldreich and Dana Ron.
\newblock Property testing in bounded degree graphs.
\newblock In \emph{Proceedings of the twenty-ninth annual ACM symposium on
  Theory of computing}, pages 406--415. ACM, 1997.

\bibitem[Goldreich et~al.(1998)Goldreich, Goldwasser, and
  Ron]{goldreich1998property}
Oded Goldreich, Shari Goldwasser, and Dana Ron.
\newblock Property testing and its connection to learning and approximation.
\newblock \emph{Journal of the ACM (JACM)}, 45\penalty0 (4):\penalty0 653--750,
  1998.

\bibitem[Gray(1988)]{Gray:88}
Robert~M. Gray.
\newblock \emph{Probability, Random Processes, and Ergodic Properties}.
\newblock Springer Verlag, 1988.

\bibitem[Khaleghi et~al.(2016)Khaleghi, Ryabko, Mary, and
  Preux]{Khaleghi:15clust}
Azadeh Khaleghi, Daniil Ryabko, J{{\'e}}r{{\'e}}mie Mary, and Philippe Preux.
\newblock Consistent algorithms for clustering time series.
\newblock \emph{Journal of Machine Learning Research}, 17:\penalty0 1--32,
  2016.

\bibitem[Lyons and Peres(2016)]{lyons2016probability}
Russell Lyons and Yuval Peres.
\newblock \emph{Probability on trees and networks}, volume~42.
\newblock Cambridge University Press, 2016.

\bibitem[Lyons et~al.(1995)Lyons, Pemantle, and Peres]{lyons1995ergodic}
Russell Lyons, Robin Pemantle, and Yuval Peres.
\newblock Ergodic theory on galton—watson trees: speed of random walk and
  dimension of harmonic measure.
\newblock \emph{Ergodic Theory and Dynamical Systems}, 15\penalty0
  (03):\penalty0 593--619, 1995.

\bibitem[Ron(2010)]{ron2010algorithmic}
Dana Ron.
\newblock Algorithmic and analysis techniques in property testing.
\newblock \emph{Foundations and Trends in Theoretical Computer Science},
  5\penalty0 (2):\penalty0 73--205, 2010.

\bibitem[Rubinfeld and Sudan(1996)]{rubinfeld1996robust}
Ronitt Rubinfeld and Madhu Sudan.
\newblock Robust characterizations of polynomials with applications to program
  testing.
\newblock \emph{SIAM Journal on Computing}, 25\penalty0 (2):\penalty0 252--271,
  1996.

\bibitem[Ryabko and Astola(2006)]{BRyabko:06a}
B.~Ryabko and J.~Astola.
\newblock Universal codes as a basis for time series testing.
\newblock \emph{Statistical Methodology}, 3:\penalty0 375--397, 2006.

\bibitem[Ryabko et~al.(2006)Ryabko, Astola, and Gammerman]{BRyabko:06b}
B.~Ryabko, J.~Astola, and A.~Gammerman.
\newblock Application of {K}olmogorov complexity and universal codes to
  identity testing and nonparametric testing of serial independence for time
  series.
\newblock \emph{Theoretical Computer Science}, 359:\penalty0 440--448, 2006.

\bibitem[Ryabko(1988)]{BRyabko:88}
Boris Ryabko.
\newblock Prediction of random sequences and universal coding.
\newblock \emph{Problems of Information Transmission}, 24:\penalty0 87--96,
  1988.

\bibitem[Ryabko(2010{\natexlab{a}})]{Ryabko:10clust}
Daniil Ryabko.
\newblock Clustering processes.
\newblock In \emph{Proc. the 27th International Conference on Machine Learning
  (ICML 2010)}, pages 919--926, Haifa, Israel, 2010{\natexlab{a}}.

\bibitem[Ryabko(2010{\natexlab{b}})]{Ryabko:10discr}
Daniil Ryabko.
\newblock Discrimination between {B}-processes is impossible.
\newblock \emph{Journal of Theoretical Probability}, 23\penalty0 (2):\penalty0
  565--575, 2010{\natexlab{b}}.

\bibitem[Ryabko(2012)]{Ryabko:121c}
Daniil Ryabko.
\newblock Testing composite hypotheses about discrete ergodic processes.
\newblock \emph{Test}, 21\penalty0 (2):\penalty0 317--329, 2012.

\bibitem[Ryabko(2017)]{Ryabko:17clin}
Daniil Ryabko.
\newblock Independence clustering (without a matrix).
\newblock \emph{arxiv.org CoRR}, 1703.06700, 2017.

\bibitem[Zhang et~al.(2014)Zhang, Yang, and Kieffer]{zhang2014universal}
Jie Zhang, En-Hui Yang, and John~C Kieffer.
\newblock A universal grammar-based code for lossless compression of binary
  trees.
\newblock \emph{IEEE Transactions on Information Theory}, 60\penalty0
  (3):\penalty0 1373--1386, 2014.

\end{thebibliography}
\end{document}